\documentclass[runningheads]{llncs}
\usepackage{booktabs}
\usepackage{array}
\usepackage{ragged2e}
\usepackage{multirow}
\usepackage{hyperref}
\usepackage{amsmath}
\usepackage{amssymb}
\usepackage[ruled,vlined]{algorithm2e}
\usepackage{subcaption}
\usepackage[T1]{fontenc}
\usepackage{graphicx}
\begin{document}
\title{Exact Synthetic Populations for Scalable Societal and Market Modeling}
\author{
Thierry Petit \and
Arnault Pachot
}

\authorrunning{Thierry Petit \and Arnault Pachot}
\institute{
Emotia, France \\
27 rue Marbeuf, 75008 Paris. \\ STATION F, 5 Parvis Alan Turing, 75013 Paris.
\email{\{thierry.petit,arnault.pachot\}@pollitics.com}
}

\maketitle              % typeset the header of the contribution
\begin{abstract}
We introduce a constraint-programming framework for generating synthetic populations that reproduce target statistics with high precision while enforcing full individual consistency. Unlike data-driven approaches that infer distributions from samples, our method directly encodes aggregated statistics and structural relations, enabling exact control of demographic profiles without requiring any microdata. We validate the approach on official demographic sources and study the impact of distributional deviations on downstream analyses. This work is conducted within the Pollitics project developed by Emotia, where synthetic populations can be queried through large language models to model societal behaviors, explore market and policy scenarios, and provide reproducible decision-grade insights without personal data. 
\keywords{Constraint Programming  \and Synthetic Populations \and Polls.}
\end{abstract}
\begin{sloppypar}
\section{Introduction}
This article presents an original method for Synthetic Population Generation (SPG) based on Constraint Programming (CP), designed to enforce both global distributional targets and individual-level coherence.

SPG is generally organized into two main families of methods \cite{chapuis2022survey}, Synthetic Reconstruction (SR), and Heuristic Combinatorial Optimization (HCO).

SR methods generate individuals by sampling attributes from marginal or reconstructed joint distributions. Classical SR includes iterative proportional fitting and iterative proportional updating schemes, and Monte Carlo based approaches that infer a joint distribution before sampling \cite{casati2015}. Markov models and probabilistic graphical models have also been used, including Bayesian networks \cite{sun2015}. Copula based methods reconstruct dependence structures \cite{jeong2016,jutrasdube2024}. Recent work applies deep generative models, such as autoencoders and variational autoencoders for rare feature combinations \cite{garrido2020}, and generative adversarial networks (GANs) for tabular synthesis surveyed in \cite{figueira2022}. Hybrid neural approaches with differentiable constraints have also been explored \cite{stoian2024}.

HCO methods construct synthetic populations by selecting or recombining individuals from a microdata sample so that the resulting distribution approximates known marginals. Foundational contributions formulate the task as an optimisation problem measured by a discrepancy between synthetic and target distributions \cite{williamson1998,voas2000}. Variants mainly differ in their search strategies, including hill climbing, simulated annealing, and genetic algorithms. 

SR and HCO methods typically require either microdata samples or training sets from which joint distributions can be learned, or alternatively detailed multiway contingency tables that are rarely fully available in public statistics. Rare contributions explore sample free optimisation starting from artificial individuals~\cite{barthelemy2012} but as far as we now, all rely on fitness functions defined with respect to target distributions and do not offer guarantee of internal coherence, since individual consistency emerges indirectly from optimisation or sampling rather than from explicit constraints.
Therefore, existing methods not well-suited to the operational setting of Pollitics, where we construct digital twins of individuals, companies, or training centers to support direct querying or economic simulations. First, we rarely have access to reliable, openly available real microdata. Second, the information we do have consists of precise statistics expressed as percentages that must be matched exactly. Third, we require a declarative approach to enforce strict individual level coherence (for example, no retired minors). Our synthetic agents are subsequently queried by Large Language Models to provide justifications and behavioural explanations, rather than being used solely through nomenclature mapping tables or analytical models.
 
CP offers a declarative framework in which structural constraints are enforced explicitly and systematically, independently of any optimisation objective. A CP based generator must nevertheless resolve several modelling challenges:
(i) The enforcement of marginal or joint statistical distributions such as age pyramids, gender ratios, or income brackets. (ii) The encoding of internal coherence rules for each individual and the optional promotion of value diversity in the population.\footnote{This constraint is optional because diversity for a given characteristic can be added during post processing, although including it directly in the model is usually easier}~(iii) The ability to scale to hundreds of individuals defined by multiple categorical attributes, which requires a model consistent with a batched solving strategy.
In this paper, we introduce and evaluate a novel CP driven method that addresses these challenges. Our code has been used in various contexts, including a polling MVP and an application that required generating more than 55000 synthetic individuals across 570 municipalities for fine grained territorial economic modelling.

\section{Constraint-based Population Generation}

Our approach formulates population generation as a constraint programming (CP) optimization problem. We construct a set of individuals from a declarative constraint model. These individuals jointly satisfy three types of constraints:

\begin{enumerate}
    \item \textbf{Exact compliance with target distributions.}  
    A global distribution constraint ensures that the generated population matches the given statistical targets exactly. 

    \item \textbf{Local logical coherence of individuals.}  
    Each individual must satisfy a set of local constraints that encode admissible combinations of attribute values, e.g., no retired chlid.

    \item \textbf{Optional structural constraints.}  
    Additional constraints (such as diversity requirements or structural restrictions on subgroups).
\end{enumerate}

We impose no restriction on the constraints attached to each individual, in particular we do not assume any tractable structural property of the underlying constraint hypergraph, such as Berge-acyclicity~\cite{DBLP:journals/constraints/BeldiceanuCDP05}.
The central difficulty is to scale. Our solution is to generate individuals by batches, which requires defining constraints that satisfy a suitable \textit{monotonicity} property so that feasibility is preserved as the population grows.
We use the following terminology: we consider a set of categorical \textit{features} (or \textit{attributes}) $\mathcal{F}$ indexed by $f \in \{1,\dots,F\}$.  
Each feature $f$ is associated with a finite domain $D_f$ representing its possible categories.
For a population of size $N$, the value of attribute $f$ for individual $i$ is represented by 
a CP variable $x^i_f$, and each such variable takes values in $D_f$. 
The set of variables 
\(
X_f = \{ x^1_f, \dots, x^N_f \}
\)
thus represents all instantiations of attribute $f$ across the population, while an individual is the tuple
\(
x^i = (x^i_1, \dots, x^i_F).
\)

\subsection{Matching Target Distributions: A Motivating Example}
In human surveys, the variance of an estimated proportion $\hat{p}$ (where $\hat{p}$ denotes the empirical frequency of a category, $p$ its true population proportion, and $n$ the number of respondents) follows the classical binomial form $v(\hat{p}) = p(1-p)/n$. Structural biases must often be corrected because many characteristics of the sample cannot be controlled. In contrast, synthetic populations can be constructed to match known distributions directly, rather than relying on large sample sizes $n$ to reduce variance.
To illustrate this property, consider two categorical attributes: age ($X_1$ with associated domain $D_1 = \{0,1,2,3\}$) and location ($X_2$, domain $D_2=\{0,1,2,3\}$), with the allocation shown in Table~\ref{tab:1} for $N=100{,}000$ individuals.
\begin{table}[h!]
\centering
\setlength{\tabcolsep}{6pt}
\scriptsize
\begin{tabular}{lccccc}
        & \multicolumn{5}{c}{$D_2$} \\
        &         & $0$ & $1$ & $2$ & $3$ \\ 
\midrule
\multirow{4}{*}{$D_1$} 
        & $0$ & 7000  & 15750 & 1750  & 10500 \\
        & $1$ & 7000  & 15750 & 1750  & 10500 \\
        & $2$ & 4000  & 9000  & 1000  & 6000  \\
        & $3$ & 2000  & 4500  & 500   & 3000  \\
\end{tabular}

\caption{Example of a categorical population distribution. }
\label{tab:1}
\end{table}

We consider a simple voting model where each individual chooses among A, B, or DK, with probabilities conditioned on age and location (Table~\ref{tab:two_side_by_side}). 
\begin{table}[htbp]
\centering
\setlength{\tabcolsep}{3pt} % compact horizontal spacing
\scriptsize% reduce font size for LNCS width
\begin{minipage}[t]{0.48\textwidth}
\centering
\begin{tabular}{lcccc}
$X_1$ (age) & $0$ & $1$ & $2$ & $3$ \\
\midrule
vote\_A & 0.45 & 0.25 & 0.25 & 0.05 \\
vote\_B & 0.25 & 0.35 & 0.55 & 0.85 \\
vote\_DK & 0.30 & 0.40 & 0.20 & 0.10 \\
\end{tabular}
\end{minipage}
\hfill
\begin{minipage}[t]{0.48\textwidth}
\centering
\begin{tabular}{lcccc}
$X_2$ (location) & $0$ & $1$ & $2$ & $3$ \\
\midrule
vote\_A & 0.25 & 0.25 & 0.45 & 0.35 \\
vote\_B & 0.65 & 0.15 & 0.15 & 0.45 \\
vote\_DK & 0.10 & 0.60 & 0.40 & 0.20 \\
\end{tabular}
\end{minipage}

\caption{Vote probabilities conditioned on age (left) and location (right).}
\label{tab:two_side_by_side}
\end{table}

These ground-truth distributions yield reference proportions A=29.5\%, B=37.27\%, DK=33.25\%. We then simulate synthetic populations under different scenarios:  
(i) random assignment of $x^i_1, x^i_2$;  
(ii) matching only age marginals;  
(iii) matching only location marginals;  
(iv) matching both;  
(v) matching perturbed marginals within $\pm k\%$.
\begin{figure}[h]
\centering
\includegraphics[width=10cm]{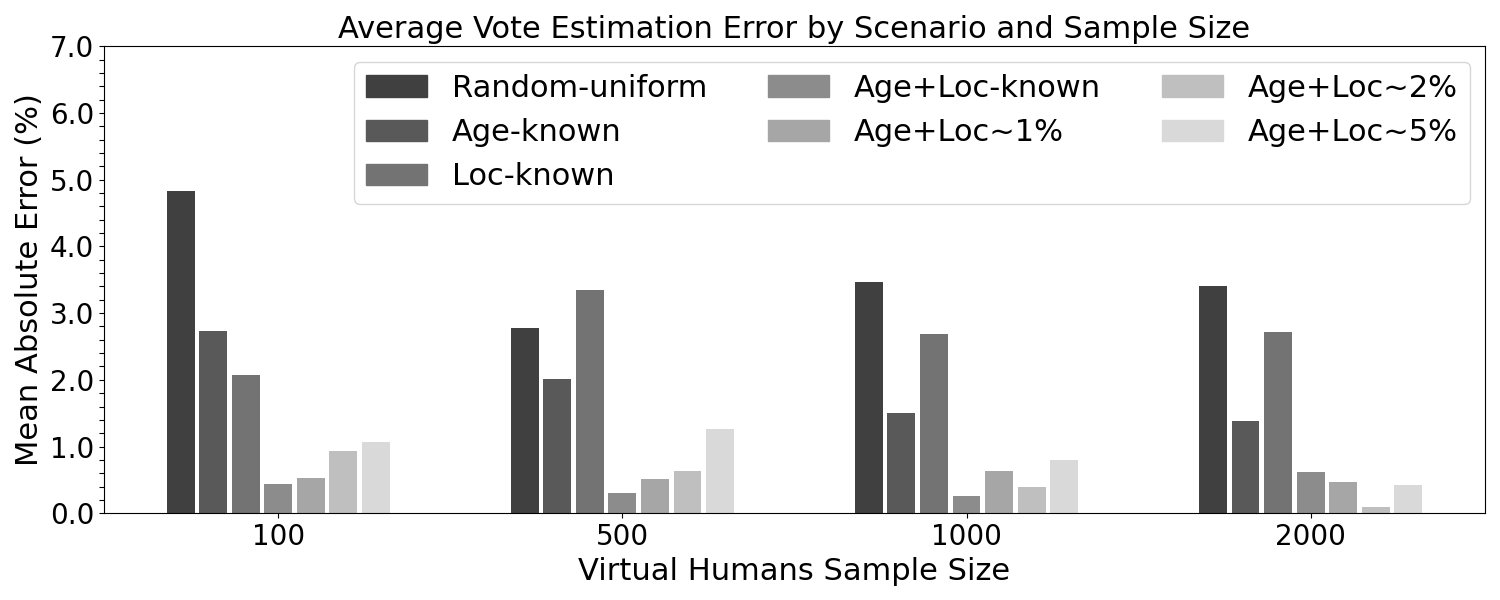}
\caption{Absolute error between estimated and true vote proportions (A, B, DK), as a function of sample size. Each group of three bars represents the error for vote A, B, and DK in a specific scenario, compared with the true results over $N$=100,000 individuals. 
%The scenarios vary in the amount of known structural information, i.e., age and location distributions in $p$.
}
\label{fig:vote_error}
\end{figure}

We assume a perfect ground-truth voting model: each individual casts a vote by sampling exactly from the conditional probabilities in Table~\ref{tab:two_side_by_side}. We compute each individual's voting probabilities as the arithmetic mean of the two values. Figure~\ref{fig:vote_error} shows that estimates are strongly biased when $x^i_1$ and $x^i_2$ do not follow their target distributions. Matching both marginals eliminates almost all error, and even slightly perturbed marginals yield highly accurate estimates.

\subsection{Batch-based Generation}
\subsubsection{Independent Features.}
We distinguish a special class of features, denoted \( \mathcal{F}_{\mathrm{indep}} \subseteq \mathcal{F} \), 
called \emph{independent features}. 
An independent feature is defined as one that is not subject to any constraints, neither globally nor through individual-level restrictions, such as persona first names.
Its values can be generated randomly from its domain before solving the constraint model, using a {sampling without replacement} strategy to promote diversity. 
\subsubsection{Generation Algorithm.}
We design a batch-based generation process where individual-level constraints are restricted to the variables of the current batch, whereas global feature constraints are stated over all known values of the corresponding features, thus incorporating the constants derived from previously generated batches. Recall that all individuals in a given population take their characteristic values from the same set of domains. 
\begin{definition}[Abstract Constraints]
We denote by $\mathcal{C}^\ast$ the set of \emph{abstract constraints} stated at the schema level 
such that:
\(
\mathcal{C}^\ast = 
\mathcal{C}^\ast_{\mathrm{feat}} \;\cup\; \mathcal{C}^\ast_{\mathrm{ind}}, 
\mathcal{C}^\ast_{\mathrm{feat}} \cap \mathcal{C}^\ast_{\mathrm{ind}} = \emptyset.
\)
\begin{itemize}
    \item \textit{Feature-level (vertical) constraints}: each $C^\ast \in \mathcal{C}^\ast_{\mathrm{feat}}$ 
    is instantiated on all variables of the features in its scope.
    \item \textit{Individual (horizontal) constraints}: each $C^\ast \in \mathcal{C}^\ast_{\mathrm{ind}}$ 
    is instantiated on a subset of variables of every individual.
\end{itemize}
\end{definition}

{\small
\begin{algorithm}[h]
\caption{Batch-Based Generation}
\KwIn{Nb. of batches $B$, batch size $n$, features $\mathcal{F}$, $\mathcal{C}^\ast$}
$\mathcal{F} \gets \mathcal{F} \setminus \mathcal{F}_{\mathrm{indep}}$ \;
$\mathcal{P} \gets \emptyset$\;
\For{$b = 1$ to $B$}{
    $\mathcal{M} \gets$ new constraint model\;
    \For{each individual $i = 1$ to $n$}{
        State $X = \{x^{1}_i, x^{2}_i, \dots, x^{f}_i\}$ in $\mathcal{M}$\;
        \For{each $C^\ast \in \mathcal{C}^\ast_{\mathrm{ind}}$}{
            Instantiate $C(Y)$ with $Y \subseteq \{x^{1}_i,\dots,x^{F}_i\}$\;
            Add $C(Y)$ to $\mathcal{M}$\;
        }
        Add constraint $[(x^{1}_i, \dots, x^{F}_i) \notin \mathcal{P}]$ to $\mathcal{M}$\;
    }
    \For{each $C^\ast \in \mathcal{C}^\ast_{\mathrm{feat}}$}{
        $Y \leftarrow \emptyset$\;
        \For{each feature $f$ involved in $C^\ast$}{
            $X_f \gets \{ x^{f}_1, \dots, x^{f}_n \} \cup \{ x^{f}_k \mid k \in \mathcal{P} \}$\;
            $Y \leftarrow Y \cup X_f$\;
            }
        Instantiate $C(Y)$\;
        Add $C(Y)$ to $\mathcal{M}$\;
    }
    $S \gets$ Solve $\mathcal{M}$\;
    $\mathcal{P} \gets \mathcal{P} \cup S$\;
}
\textbf{Return} $\mathcal{P}$\;
\end{algorithm}
}

\subsection{Feature-Level Constraints.}
The essential constraint for accurately simulating a population is the enforcement of 
target distributions on each feature. We call this set of global constraints the \textit{distribution constraints}. 
In our framework, other feature-level constraints  
can be regarded either as facilities provided to the user or as constraints for improving the solving process, 
as in most cases they can be handled in post-processing. For instance, a diversity constraint can be enforced 
after the generation, by randomly selecting individuals once the categorical distribution has been 
satisfied. We therefore refer to these additional constraints as 
\textit{optional feature-level constraints}. 

\subsubsection{Distribution Constraint.}
The \emph{distribution constraint} enforces the alignment of each feature with a 
target categorical distribution. This constraint is related to global cardinality constraints, 
which are widely studied in the literature~\cite{schmied2025gcccost}. However, our approach requires constraints that are enforced as best as possible through an optimization objective rather than satisfaction, similarly to cost function networks and soft constraints~\cite{allouche2015tractability2,CooperGSSZW10,vanHoeve2011overconstrained}.

In our setting, each target percentage associated with a category can be viewed as a bin with a desired fill level. During batched generation, every individual assigned to a category contributes one unit to the corresponding bin. Each bin has a hard upper bound equal to the batch size, but a soft target defined by the prescribed distribution. As a result, the solver naturally tends to allocate individuals so as to reduce the deviation from the target fill levels whenever underfilled bins remain, or to incur at most an additional unit of overflow per individual otherwise. As shown in the experimental section, a decomposition into primitive constraints is the most effective formulation for this purpose. This approach benefits directly from solver-level explanation and learning mechanisms, such as those implemented in OR-Tools CP-SAT~\cite{ortools:cp-sat}.

\begin{proposition}[Largest Remainder Rounding]
Let $p_1,\dots,p_q \in [0,100]$ be target percentages with $\sum_{i=1}^q p_i = 100$ and let $N \in \mathbb{N}$ be the total number of individuals.
Define the (real-valued) ideal allocations
\[
f_i \;=\; \frac{p_i}{100}\,N \quad (i=1,\dots,q),
\]
their integer parts $t_i^{(0)}=\lfloor f_i \rfloor$, and fractional parts $r_i = f_i - t_i^{(0)}$.
Let
\[
R \;=\; N - \sum_{i=1}^q t_i^{(0)} \;=\; \sum_{i=1}^q r_i, \quad\text{so that}\quad 0 \le R < q.
\]
Let $S \subseteq \{1,\dots,q\}$ be the indices of the $R$ largest $r_i$ (break ties arbitrarily), and set
\(
t_i \;=\; t_i^{(0)} + \mathbf{1}_{\{i \in S\}}.
\). 
Then $\sum_{i=1}^q t_i = N$ and $|t_i - f_i| < 1$ for all $i$.
\end{proposition}

\begin{proof}
By construction, $\sum_i t_i^{(0)} \le N$ and $R = N - \sum_i t_i^{(0)}$ is an integer with $0 \le R < q$.
Adding $1$ to exactly the $R$ indices with largest $r_i$ yields
\[
\sum_{i=1}^q t_i \;=\; \sum_{i=1}^q t_i^{(0)} + R \;=\; N.
\]
Moreover, for each $i$, either $t_i = t_i^{(0)}$ giving $|t_i - f_i| = r_i < 1$, or $t_i = t_i^{(0)}+1$ giving $|t_i - f_i| = 1 - r_i < 1$.
\hfill $\square$
\end{proof}

\begin{definition}[Distribution Constraint]\label{def:distribution_constraint}
Let $\mathit{obj}_f$ be a variable and $X_f \in \mathcal{F}$ be the variable set of a feature discretized into $q$ disjoint bins 
$B_1,\dots,B_q$.  
Each $B_j$ is associated with a target percentage $p_j$ $(1 \le j \le q)$, from which we compute $t_j$ as the global target number of individuals to allocate in $B_j$ after generating all batches, using the largest remainder method. 

Let $e_j$ be the number of individuals already generated in $B_j$.  
For the current batch of size $n$, the constraint is expressed as:
\begin{align}
t_j^{\text{batch}} &= \max\bigl(0,\, t_j - e_j \bigr),
\tag{1} \\
b_{ij} &=
\begin{cases}
1 & \text{if } x_i^f \in B_j,\\
0 & \text{otherwise,}
\end{cases}
\tag{2} \\
\sum_{j=1}^q b_{ij} &= 1 \quad \forall i \in \{1,\dots,n\},
\tag{3} \\
c_j &= \sum_{i=1}^n b_{ij},
\tag{4} \\
\delta_j &= |\, c_j - t_j^{\text{batch}} \,|,
\tag{5} \\
\mathit{obj}_f &\geq \sum_{j=1}^q \delta_j.
\tag{6}
\end{align}

\noindent
\textbf{Explanations.}
\begin{enumerate}
    \item[(1)] $t_j^{\text{batch}}$ adjusts the global target $t_j$ to account for $e_j$.
    \item[(2)] $b_{ij}$ is a Boolean indicator: it equals $1$ if individual $i$ is placed in bin $B_j$, where $x_i^f$ denotes the value of feature $X_f$ for individual $i$.
    \item[(3)] Each individual must belong to exactly one bin.
    \item[(4)] $c_j$ counts the number of individuals assigned to bin $B_j$ in the current batch.
    \item[(5)] $\delta_j$ is the absolute deviation from the batch target $t_j^{\text{batch}}$.
    \item[(6)] Variable $\mathit{obj}_f$ is to be minimized: total deviation across all bins.
\end{enumerate}
\end{definition}

The constraint of Defintion~\ref{def:distribution_constraint} satisfies a property that allows to use it in a batch-based solving process.
For an assignment $S$ on $X$ and $Y \subseteq X$, we write $S[Y]$ for the projection of $S$ onto $Y$, defined by $S[Y](x)=S(x)$ for all $x\in Y$.
\begin{definition}[Extension-Preserving Optimality]
\label{def:extension_optimality}
Let $X$ be a set of variables and $\mathit{obj}$ a variable, and the problem
\([
\min\ \mathit{obj}\quad\text{subject to: }C(X)\text{ and }\mathit{obj}\ge C_{\mathit{obj}}(X),
]\)
where $C_{\mathit{obj}}$ is an expression over the variables in $X$
(e.g.\ $C_{\mathit{obj}}(X)=\sum_{j=1}^q \delta_j(X)$).  
Given $Y \subseteq X$, the problem is \emph{extension-preserving optimal on $Y$} if every optimal solution $S_1$ of the restricted problem on $Y$ can be extended to an optimal solution $S_2$ on $X$ such that $S_1[Y]=S_2[Y]$.
\end{definition}

\begin{proposition}
Let $X$ be the set of variables involved in the distribution model of Definition~\ref{def:distribution_constraint} for a batch of $N$ individuals, and let $Y\subseteq X$ be the subset corresponding to a sub-batch of $M<N$ individuals. The decomposition of the distribution constraint is extension-preserving optimal on $Y$ in the sense of Definition~\ref{def:extension_optimality}.
\end{proposition}
\begin{proof}
For $Y$, the 
$\delta_j(Y)$'s are computed with respect to the same global targets $t_j$ as
in the full problem on $X$ with $N$ individuals. An optimal solution on
$Y$ minimizes $\sum_j |c_j(Y)-t_j|$.
Since $\sum_{j=1}^q b_{ij} = 1$ (Definition~\ref{def:distribution_constraint}), any global improvement of the objective must come
from a strictly better combination of bin counts with respect to the same
targets $t_j$. Therefore, any global solution with strictly smaller objective would induce a strictly
smaller value of $\sum_j |c_j(Y)-t_j|$ on $Y$, contradicting the optimality
of the solution on $Y$. \hfill $\square$
\end{proof}
The extension-preserving property holds for the distribution constraint in
isolation, but individual coherence constraints and additional global
constraints (such as diversity) may alter this behaviour. We measure
this impact empirically in the experimental section.

\subsubsection{Optional Feature-Level Constraints.}
Although our system does not limit the global constraints that can be implemented on 
features, we implemented a declarative API in which 
all constraints are expressed through a JSON specification. 
This format supports several \emph{optional feature-level constraints}, 
including the classical \textit{AllDifferent} constraint~\cite{reg94}, and a dedicated \textit{diversity constraint}. 
We focus on describing this diversity constraint, derived from 
CP existing approaches~\cite{Hebrard2005,Hebrard2007,PetitTrapp2019,IngmarDeLaBandaStuckeyTack2020}. 
This constraint is useful when no specific distribution is known for a feature, such as first names, 
or preferences from which no statistical data is available. 
\begin{definition}[Diversity Constraint]\label{def:diversity_constraint}
Let $X_f$ be a feature, and let $I_{\mathrm{exist}}$ and $I_{\mathrm{batch}}$ be
the index sets of already generated individuals and of the current batch,
respectively. For any distinct $i_1,i_2 \in I_{\mathrm{exist}} \cup I_{\mathrm{batch}}$
with $i_1<i_2$, define
\[
b_{i_1i_2} =
\begin{cases}
0 & \text{if } x^f_{i_1} \ne x^f_{i_2},\\
1 & \text{if } x^f_{i_1} = x^f_{i_2}.
\end{cases}
\]
The diversity objective is expressed as
\[
\mathit{obj} \;\ge\;
\sum_{\substack{i_1<i_2\\ i_1,i_2 \in I_{\mathrm{exist}} \cup I_{\mathrm{batch}}}} b_{i_1i_2}.
\]
\noindent
\textbf{Explanation.}
Minimizing $\mathit{obj}$ penalizes identical feature values.
\end{definition}
\begin{proposition}
Let $X$ be the set of variables involved in the diversity constraint of
Definition~\ref{def:diversity_constraint} for a batch of $N$ individuals,
and let $Y \subseteq X$ be the subset corresponding to a sub-batch of
$M < N$ individuals.  
The diversity constraint is extension-preserving
optimal on $Y$ in the sense of
Definition~\ref{def:extension_optimality}.
\end{proposition}

\begin{proof}
The diversity objective is $\mathit{obj}_f \ge \sum_{i_1<i_2} b_{i_1 i_2}$,
where each $b_{i_1 i_2}$ depends only on the pair of individuals
$(i_1,i_2)$ considered. The variables $b_{i_1 i_2}$ involving indices in
$X\setminus Y$ are independent of all variables indexed in $Y$. 
As all variables in $X_f$ have the same initial domain,
any optimal assignment on $Y$ can be extended by optimally choosing the
$b_{i_1 i_2}$ for pairs in $X\setminus Y$, without affecting the objective
contribution of $Y$. Such an extension attains the global optimum.
\hfill$\square$
\end{proof}

\subsubsection{Handling Interdependent Distributions.}
In many applications, features exhibit statistical dependencies. For example,
voting intentions may vary across age categories, so that the target
distribution of a ``vote'' feature $X_g$ depends on the category of an ``age''
feature $X_f$. Two modelling strategies can be used.

\begin{itemize}
    \item \textit{Two-phase generation (sequential case).}
    If the dependency structure between features is acyclic and known (e.g.\ the
    distribution of $X_g$ depends only on $X_f$, or more generally if a
    topological ordering exists), we first generate all individuals by
    enforcing the distribution constraint on the parent features. For each
    category $v$ of a parent feature $X_f$, we then generate the dependent
    feature $X_g$ in the subpopulation $\{i : x^i_f = v\}$ using its own
    distribution constraint.
    
    \item \textit{Joint generation (cyclic or mutually dependent case).}
    When several features mutually constrain each other, we introduce a
\emph{joint feature} whose domain is the Cartesian product of their domains:
\[
D_{\mathrm{joint}}
= \{(v_{f_1},\dots,v_{f_F}) \mid v_f \in D_f\}.
\]
A single distribution constraint is then applied to this joint feature.

\end{itemize}

When detailed joint statistics for $D_{\mathrm{joint}}$ are available (e.g.\
cross-tabulated survey or census data), each joint category
$(v_{f_1},\dots,v_{f_F})$ is assigned its target percentage.  
Otherwise, only the known joint statistics are assigned, and an approximate
distribution for the remaining combinations can be constructed from the
available marginals using a domain-specific rule or an independence
assumption, which must be explicitly stated.  
In all cases, target counts for the joint categories are obtained using the
largest remainder method, and the standard distribution constraint of
Definition~\ref{def:distribution_constraint} is applied to the joint feature.

\subsection{Individual Constraints}
The constraints ensuring the internal consistency of each generated individual typically 
express dependencies or interactions between attributes, 
such as a relation between age and professional activity, or between a city and its administrative department. 
Formally, they can be modeled using standard logical, Boolean, 
and table constraints for specifying the set of allowed or forbidden tuples. 

\section{Applications}

\subsection{Synthetic Population Generation}
The experiments were run on an Apple M2 Pro (16 GB RAM, macOS Ventura 13.5) with Python 3.11 and OR‑Tools 9.14.
To interpret these experiments properly, it is important to consider the intended use case of the framework, designed for SPG on demand, through JSON specifications possibly dynamically refreshed. We fixed a 30-second solving time limit to ensure responsive user-interaction. 

\paragraph{Trade-offs in Constrained Generation.}
This experiment evaluates how the presence and structure of individual constraints affects the system's ability to match global targets. 
The MAPE is calculated as follows: for each feature \( f \) with a target distribution defined by \( k \) bins, the target percentages \( p_i \) (where \( i = 1, \ldots, k \) and \( \sum p_i = 100\% \)) are compared against the actual percentages \( q_i \) derived from the generated population. The MAPE is the mean of the absolute errors \( |p_i - q_i| \), expressed as a percentage. To handle cases where target percentages might be zero (avoiding division by zero), we use absolute difference normalized by the total count, averaged across bins:
\(
\text{MAPE}_f = \frac{100}{n} \sum_{i=1}^{n} |p_i - q_i|.
\)
The MAPE for an instance is the average of \( \text{MAPE}_f \) across all features. 

The experiment generates synthetic instances of 100 individuals with 15 features per individual, each with a distribution constraint on a number of beans ranging from 2 to 15, randomly assigned per instance and feature. The number of individual constraints ranges from from 4 to 20, incremented by 4 (i.e., [4, 8, 12, 16, 20]).
The proportion of forbidden tuples in each constraint, set to [10\%, 20\%, 30\%]. While it might be possible for one constraint to lie near the phase transition, considering multiple compatibility constraints all positioned at this critical hardness is not representative of practical cases of human populations. 
Figure~\label{fig:combined} shows two graphs that plot average MAPE against batch size (in [1, 5, 10, 50]) in the generation process, one for binary individual constraints (two features) and one for ternary individual constraints (three features). 

\begin{figure}[h]
    \centering
    \begin{subfigure}[b]{0.48\textwidth}
        \includegraphics[width=\textwidth, height=0.23\textheight]{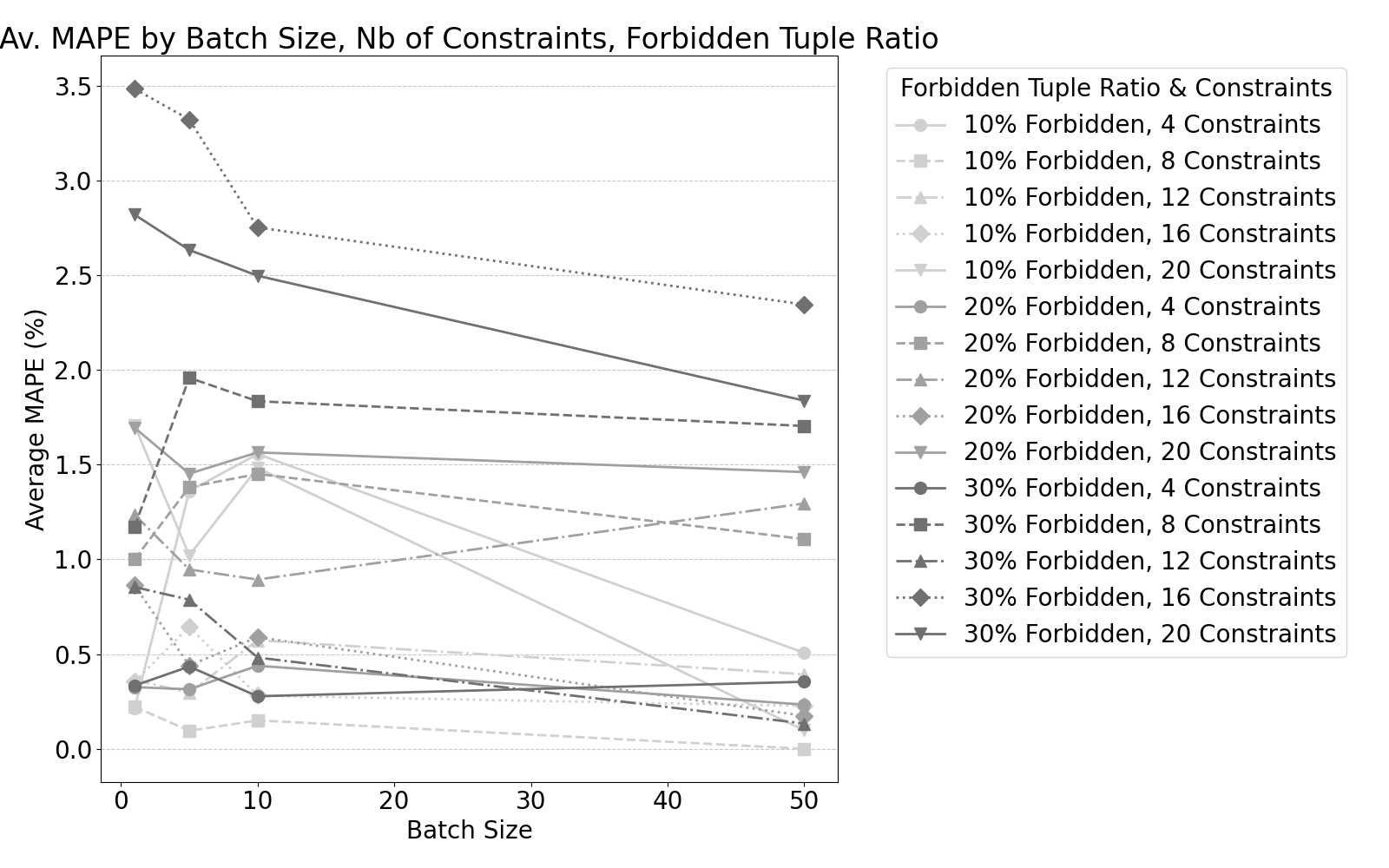}
        \caption{Binary Constraints}
        \label{fig:binary}
    \end{subfigure}
    \hfill
    \begin{subfigure}[b]{0.48\textwidth}
        \includegraphics[width=\textwidth, height=0.23\textheight]{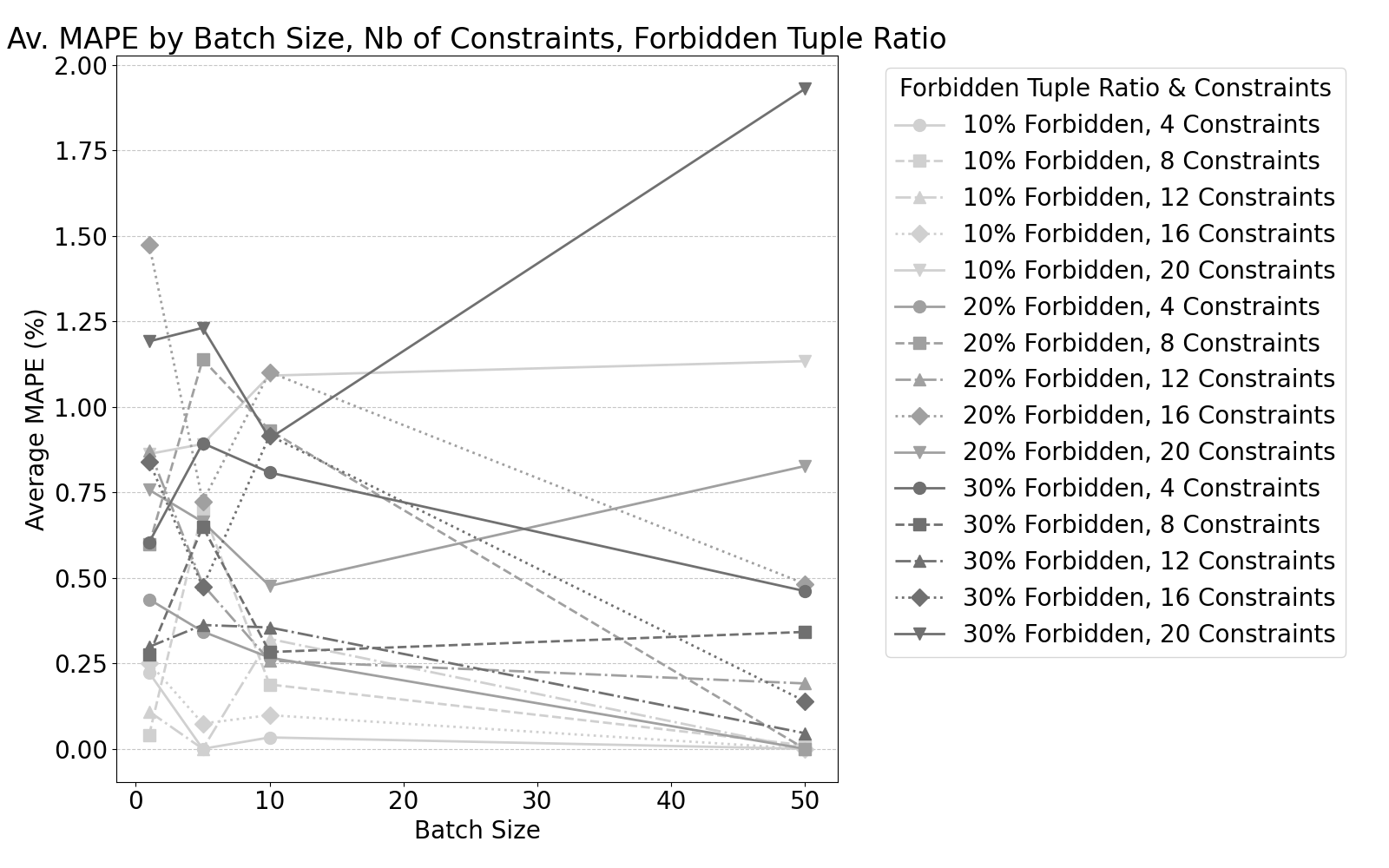}
        \caption{Ternary Constraints}
        \label{fig:ternary}
    \end{subfigure}
    \caption{Average MAPE by batch size, number of constraints and tightness.}
    \label{fig:combined}
\end{figure}

Both plots reveal the trade-off between constraint enforcement and distribution accuracy: higher constraint counts logically tend to increase MAPE, particularly at smaller batch sizes, indicating greater difficulty in matching global targets. 
Increasing values for the largest batch size in the ternary case are due to the imposed 30-second resolution limit (with a greater time limit we observed the same trend as for other instances).

To complete the experiment, we build on the setup of Trade-offs in Constrained Generation to assess scalability and the effect of batch size on performance. Table~\ref{tab:timings} shows the execution time per batch for the median case of 12 constraints and 20\% of forbidden tuples. 
\begin{table}[h!]
\scriptsize
\centering
\setlength{\tabcolsep}{4pt}
\renewcommand{\arraystretch}{1.1}
\begin{tabular}{cccc}
\textbf{Pop. Size} & \textbf{Batch Size} & \textbf{Av. Time (s)} & \textbf{\% Optimal batches} \\
\midrule
100   & 1  & 0.60 & 100.0 \\
100   & 5  & 0.73 & 100.0 \\
100   & 10 & 0.80 & 100.0 \\
1000  & 1  & 4.59 & 100.0 \\
1000  & 5  & 4.76 & 100.0 \\
1000  & 10 & 4.85 & 100.0 \\
5000  & 1  & 12.55 & 100.0 \\
5000  & 5  & 15.12 & 99.94 \\
5000  & 10 & 19.85 & 98.02 \\
\end{tabular}
\caption{Average runtime per batch and percentage of locally optimal batches, for different population and batch sizes.}
\label{tab:timings}
\end{table}

We observe a moderate constant multiplicative increase in runtime as the population size grows. 5000 individuals is the limit for solving all  batches to their local optimum in less than 30 seconds when the batch size is in [1,10]. 
%%%%%%%%%%%%%%%%%%%%%%%%%%%%%%%%%%%
%%%%%%%%
\paragraph{ District Demographic Simulation.}
We evaluate our method on five interdependent demographic features for the
Mulhouse district (France): age, area, gender, employment, and political
orientation. Target distributions come from INSEE (age, gender, employment,
area) and electoral sources (ideology). The categorical domains are as follows:
age has seven groups (0--18, 19--30, 31--40, 41--50, 51--65, 66--75, 76+);
area contains seven municipalities (Mulhouse, Riedisheim, Illzach,
Brunstatt-Didenheim, Pfastatt, Sausheim, Lutterbach); gender has three
categories (male, female, other); employment has five (student, unemployed,
retired, employed, self-employed/director); ideology has five (left,
center-left, center-right, right, unknown).

Two individual constraints are enforced: individuals aged 0--18 cannot be
retired, and their political orientation must be ``unknown''. Inter-feature
dependencies are represented using composite features treated as new
variables with their own target distributions: $(X_2,X_1)$ (age by area),
$(X_1,X_4)$ (age by employment), and $(X_3,X_5)$ (gender by ideology).
Consistency between base and composite variables is enforced through
allowed-tuple relations.

\begin{table}[h]
\scriptsize
\centering
\begin{tabular}{l|c|c|c|c|c|c|c|c}
\textbf{MAPE (\%)} 
& $X_1$ & $X_2$ & $X_3$ & $X_4$ & $X_5$ 
& $(X_1,X_4)$ & $(X_2,X_1)$ & $(X_3,X_5)$ \\
\midrule
Value 
& 4.9 & 0.0 & 0.0 & 0.0 & 0.0 
& 0.0 & 0.0 & 0.0 \\
\end{tabular}
\caption{Mean Absolute Percentage Error (MAPE) for each feature.}
\label{tab:mape}
\end{table}

All distributions are matched exactly except for age ($X_1$), which carries
the strongest set of constraints. The generated population remains close to
the real district structure. Deviations between target and generated
distributions may also reveal inconsistencies in the input statistics,
suggesting a secondary use of our framework for validating cross-distribution
coherence.

\subsection{Application Usage}
\subsubsection{Virtual Polling}

We developped Pollitics (\url{pollitics.com}), a virtual polling platform that replaces human survey panels with synthetic populations generated from official demographic statistics. 

Each virtual individual is queried independently through a Large Language Model, producing aggregated results that mimic traditional opinion polls. National populations for dozens of countries have been generated using publicly available demographic marginals (age, gender, region, etc.) from national statistical institutes. The two major advantages of using synthetic populations stem from statistical data are the absence of privacy concerns and the ability to iterate as many times as required, with new questions or specific populations. 

\begin{figure}[h]
\centering
\includegraphics[width=0.6\linewidth]{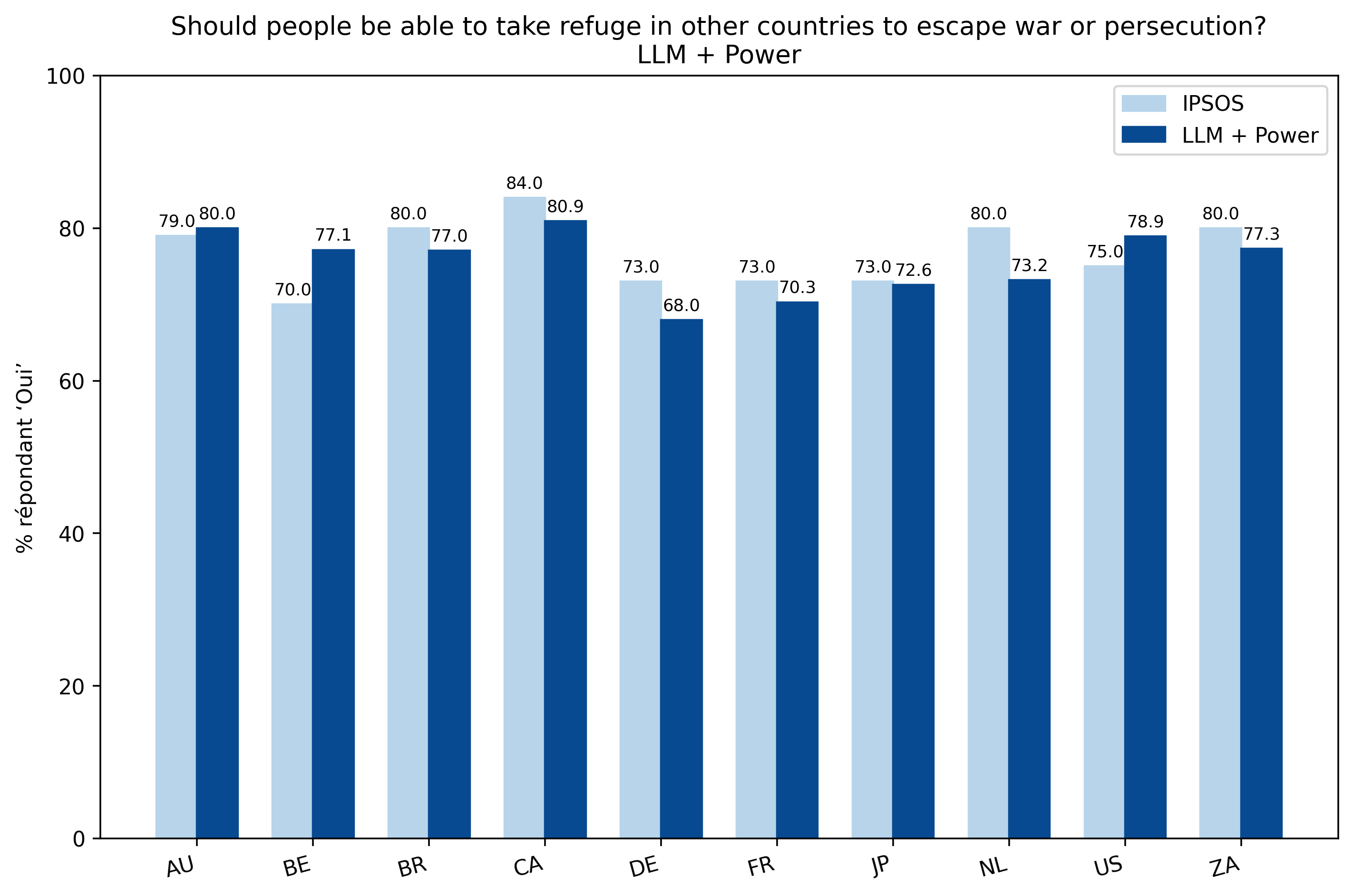}
\caption{Comparison between human IPSOS polling results and virtual polling
results obtained with synthetic populations queried via LLMs (“LLM + Power”)
on the question: \emph{“Should people be able to take refuge in other
countries to escape war or persecution?”}.}
\label{fig:refugee_rights}
\end{figure}

Once a population is instantiated, Pollitics asks each synthetic individual
the user’s question and aggregates their answers in real time. In this
application, questions are binary. For each respondent, an LLM first produces
three probabilities \(p_{\text{yes}}, p_{\text{no}}, p_{\text{dk}}\); a vote is then
sampled accordingly. Different aggregation modes can be applied.
(1) \emph{Raw mode} simulates each ballot directly from
\(p_{\text{yes}}\), with an optional rule counting cases with
\(p_{\text{yes}}>0.5\) as ``yes.'' 
(2) \emph{Gumbel mode} injects Gumbel noise into the logits of the three
probabilities, modelling the stochastic variability observed in human
respondents, who do not always act deterministically given their latent
preferences. 
(3) \emph{Power mode} sharpens the distribution by applying an exponent
\(\alpha>1\) to the probabilities below and above a given threshold, reinforcing decisive preferences and
mitigating the dampening of strong opinions induced by LLM safety filters, an
effect also noted in recommendation settings~\cite{Sinacola2025}. 
These transformations may be selected or combined depending on the desired
simulation behaviour.

For broad, non–time-sensitive questions, the system has empirically shown results close to human polls, suggesting that well-structured synthetic populations combined with LLM-based reasoning can approximate real survey outcomes. An illustrative example is shown in Figure~\ref{fig:refugee_rights} (see \url{https://pollitics.com/foundations} for other examples). 

\subsubsection{Territorial Economic Intelligence}

Beyond aggregate national polling, we built a scalable intelligence tool for French territorial and industrial monitoring based on synthetic populations. It is designed to continuously ingest new data sources and incorporate up-to-date contextual information in real time. 

In this setting, synthetic populations are generated for 570 communes in the Lyon--Saint-Étienne--Roanne area (about 100,000 digital twins), using public INSEE statistics, for  total of one hundred thousands individuals. The process combines the \emph{dossier\_complet} INSEE dataset (see \url{https://www.insee.fr/fr/statistiques/zones/2011101}), communal reference files, electoral data (\url{https://data.gouv.fr}), labour-market indicators, household and housing statistics, and cross-tabulated demographic attributes (age~$\times$~gender, formation~$\times$~gender, activity, employment categories). These distributions are automatically extracted, normalised, with a particular care on interdependent data to generate faithful synthetic inhabitants. A similar strategy is used to model \emph{legal entities}, where synthetic companies are instantiated from aggregated economic and employment indicators, including datasets provided by France Travail, enabling simulations of industrial dynamics.
\begin{figure}[h]
\centering
\includegraphics[width=1\linewidth]{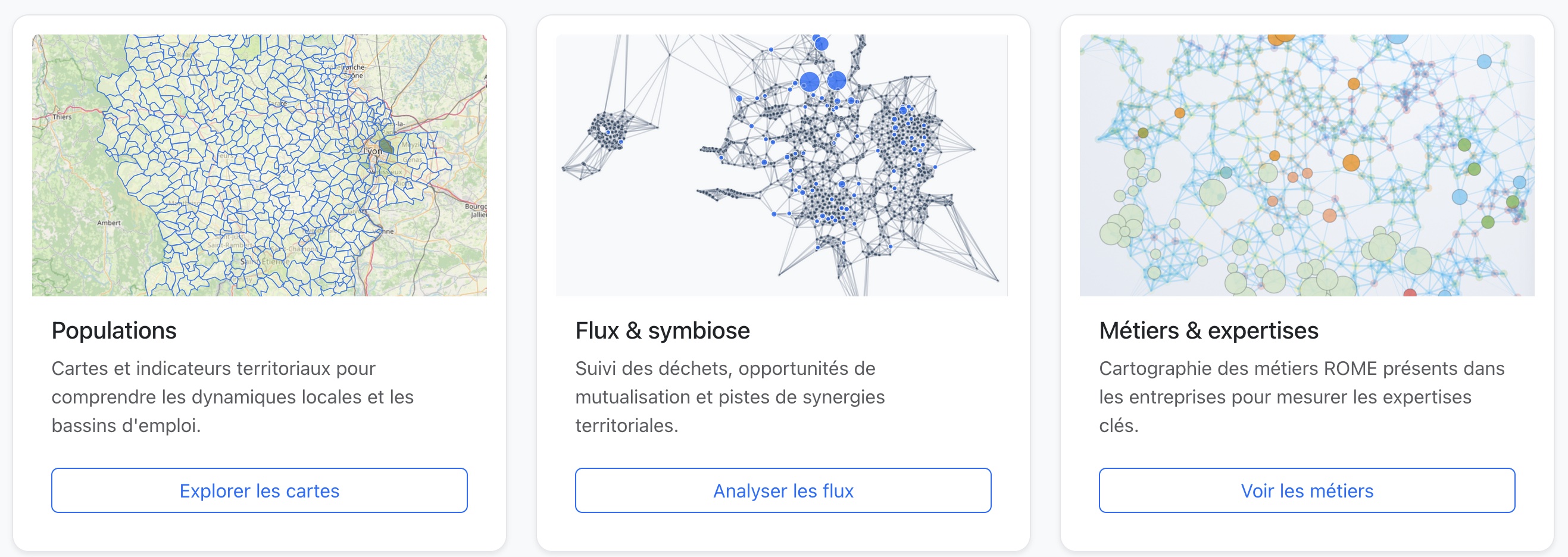}
\caption{French territorial economic intelligence tool: synthetic populations of Lyon--Saint-Étienne--Roanne districts mapped to local economic indicators.}
\label{fig:territorial_econ}
\end{figure}

These territorial populations enable applications in local economic analysis, including the estimation of a company's local image, perceived industrial attractiveness, and attitudes toward heritage preservation or circular-economy initiatives. Each virtual inhabitant represents a statistically grounded profile; the system can therefore simulate local sentiment or behavioural responses by interrogating individuals one by one through an LLM, just as in virtual polling.

The overall approach yields a granular, traceable, and updatable territorial intelligence layer, capable of supporting market studies, public-policy diagnostics, and strategic planning across hundreds of communes.

\subsubsection{AI-Driven Text Evaluation \& Message Optimization}

We used synthetic populations for AI-driven evaluation of written communication, enabling users to analyse, compare, and optimise messages for specific audiences. The system leverages large language models to assess clarity, tone, persuasion strength, emotional resonance, and potential misinterpretations. Each message is processed through a multi-criteria evaluation pipeline that provides structured feedback, highlighting strengths, weaknesses, and opportunities for improvement.

\begin{figure}[h]
\centering
\includegraphics[width=1\linewidth]{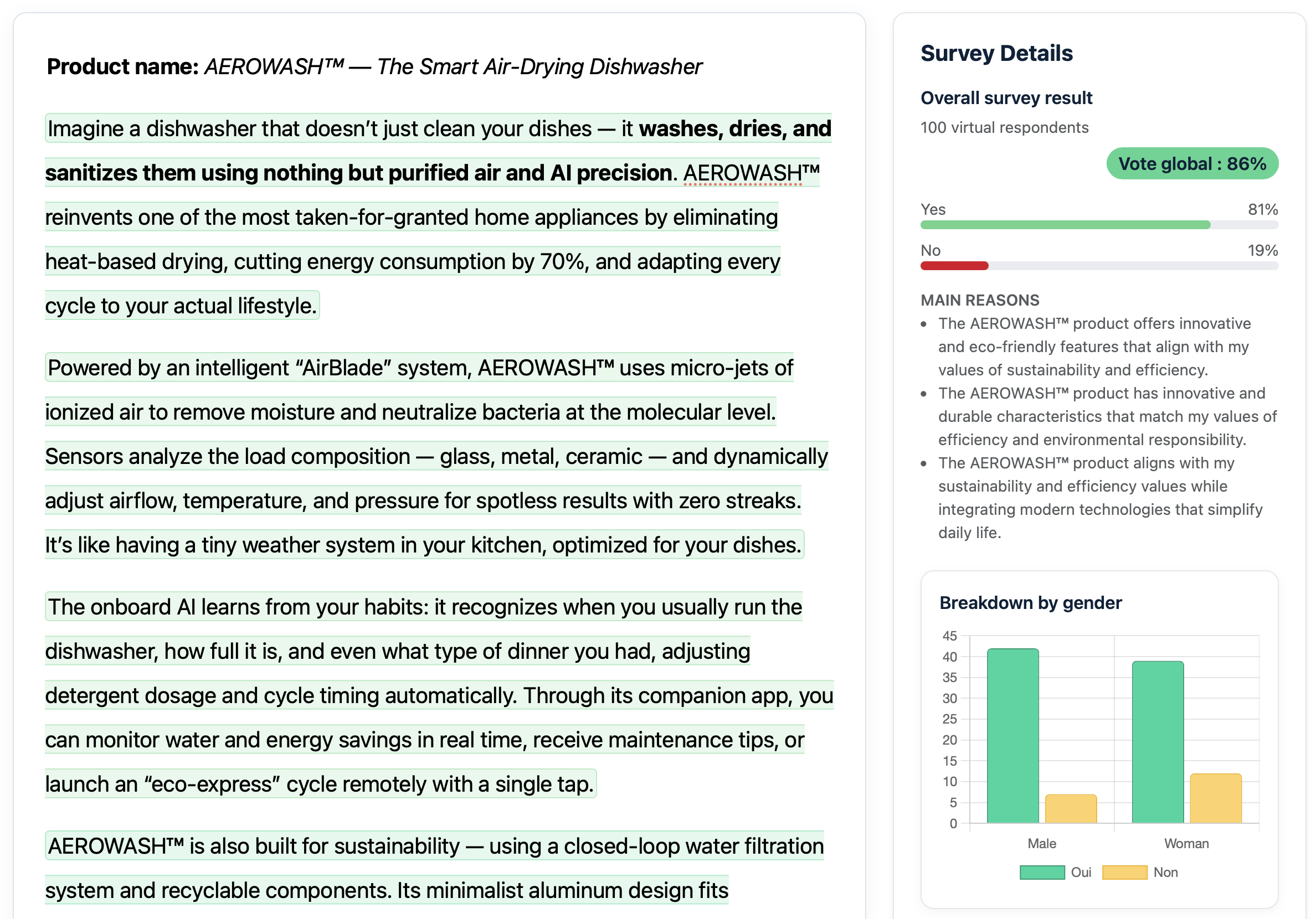}
\caption{Illustrative example of AI-driven text evaluation: multi-criteria scoring and optimisation suggestions for a target audience.}
\label{fig:text_evaluation}
\end{figure}

 Messages can be evaluated not only in absolute terms, but also \emph{as perceived by selected demographic groups}. By querying virtual individuals one by one, the system estimates how different audiences (e.g.\ young adults, retirees, industrial workers, local residents) would react to a proposed text. This enables message optimization tuned to specific territories, market segments, or socio-economic profiles.
The result is a versatile tool for refining public communication, marketing materials, policy messaging, and stakeholder engagement strategies, with rapid feedback loops and reproducible evaluations.

\section{Discussion and Conclusion}
This work introduced a Constraint Programming framework for generating exact synthetic populations. Although the model is entirely driven by aggregated statistics and does not require microdata, each individual is locally coherent thanks to explicit logical constraints, yielding structured populations that are simultaneously distribution-accurate and semantically valid. This property constitutes a distinctive property in the SPG literature, allowing us to cover a broad range of new applications. The method scales to thousands of individuals and supports a broad range of applications, including national virtual polling, territorial economic intelligence, and AI-driven text evaluation for human and corporate digital twins. 

More broadly, the results illustrate the complementary roles of constraint programming and large language models. LLMs cannot replace optimization approaches capable of tackling problems of combinatorial nature. However, they excel at synthesizing heterogeneous information, extracting weak signals, and enabling natural-language interaction with synthetic agents. This paper demonstrates a productive synergy between leveraging LLMs for the tasks they are genuinely well suited to and relying on the formal guarantees offered by a constraint-reasoning approach.

From an ethical perspective, it is essential to emphasize that such systems are decision-support tools and must not be treated as substitutes for human judgment. At the same time, they offer important positive externalities: they enable qualitative and quantitative studies for organizations that could not otherwise afford them, allow large-scale simulations with no privacy risk since no generated individual corresponds to a real person, and make it possible to represent and query minority groups that are statistically important but extremely hard to reach in traditional surveys. This provides a principled way to explore viewpoints that are structurally under-sampled in human polling. Finally, we emphasize that all questions submitted to our system are filtered through a dedicated ethical-safety layer, described in~\cite{petit2025icmlc}, ensuring that inappropriate, discriminatory, or harmful queries cannot be processed. 

As future work, we plan to combine heterogeneous statistical sources to automatically generate populations that best satisfy all available marginals; deviations in the resulting cross-distributions may highlight inconsistencies in public statistics, providing a principled approach to data validation. 
\bibliographystyle{splncs04}
\bibliography{biblio}
\end{sloppypar}
\end{document}